\documentclass[10pt, twoside,lettersize,journal]{IEEEtran}
\ifCLASSINFOpdf
\else
\fi
\usepackage[brazil, english]{babel}
\usepackage{array}
\usepackage{amsmath, amssymb, amsthm, bm, amsfonts}
\usepackage[linesnumbered, ruled]{algorithm2e}
\usepackage{caption}
\usepackage{float} 
\usepackage{subcaption}
\usepackage{textcomp}
\usepackage{stfloats}
\usepackage{verbatim}
\newtheorem{assumption}{Assumption}
\usepackage{graphicx}
\newcommand{\myspace}[1]{\par\vspace{#1\baselineskip}}
\usepackage{indentfirst}
\usepackage[brazil, english]{babel}
\usepackage[colorlinks, 
            linkcolor=blue, 
            anchorcolor=red, 
            urlcolor=blue, 
            citecolor=green
            ]{hyperref}
\usepackage{pifont}
\usepackage{makecell}

\RequirePackage{CJKnumb}
\DeclareMathSizes{10}{10}{7}{5}
\usepackage{color, url, balance, times, helvet, courier}
\usepackage[square, comma, sort&compress, numbers]{natbib}
\usepackage{booktabs, threeparttable, multirow, adjustbox, threeparttable}
\newtheorem{proposition}{Proposition}
\newtheorem{definition}{Definition}
\newtheorem{Property}{Property}
\newtheorem{theorem}{Theorem}
\newtheorem{lemma}{Lemma}
\theoremstyle{remark}
\newtheorem{remark}{Remark}
\usepackage[commandnameprefix=always]{changes}
\colorlet{Changes@Color}{red}
\usepackage{orcidlink}

\begin{document}

\title{Globally Convergent Accelerated Algorithms for Multilinear Sparse Logistic Regression with $\ell_0$-constraints}
\author{
Weifeng Yang\orcidlink{0009-0004-9453-3525} and Wenwen Min$^*$\orcidlink{0000-0002-2558-2911} 
\IEEEcompsocitemizethanks{
\IEEEcompsocthanksitem 
Weifeng Yang and Wenwen Min contributed equally to this work, and they are with the School of Information Science and Engineering, Yunnan University, Kunming 650091, Yunnan, China. 
E-mail: minwenwen@ynu.edu.cn. 
}
\thanks{Manuscript received XX, 202x; revised XX, 202x. \newline
        (Corresponding authors:  Wenwen Min)}}
        
\markboth{IEEE TRANSACTIONS ON NEURAL NETWORKS AND LEARNING SYSTEMS, 2023}{Yang \MakeLowercase{\textit{et al. }}: Accelerated Algorithm for Multilinear Sparse Logistic Regression}

\maketitle
\begin{abstract}

Tensor data represents a multidimensional array. Regression methods based on low-rank tensor decomposition leverage structural information to reduce the parameter count. Multilinear logistic regression serves as a powerful tool for the analysis of multidimensional data. To improve its efficacy and interpretability, we present a Multilinear Sparse Logistic Regression model with $\ell_0$-constraints ($\ell_0$-MLSR).
In contrast to the $\ell_1$-norm and $\ell_2$-norm, the $\ell_0$-norm constraint is better suited for feature selection. However, due to its nonconvex and nonsmooth properties, solving it is challenging and convergence guarantees are lacking. 
Additionally, the multilinear operation in $\ell_0$-MLSR also brings non-convexity.
To tackle these challenges, we propose an Accelerated Proximal Alternating Linearized Minimization with Adaptive Momentum (APALM$^+$) method to solve the $\ell_0$-MLSR model. 
We provide a proof that APALM$^+$ can ensure the convergence of the objective function of $\ell_0$-MLSR. 
We also demonstrate that APALM$^+$ is globally convergent to a first-order critical point as well as establish convergence rate by using the Kurdyka–Łojasiewicz property.  
Empirical results obtained from synthetic and real-world datasets validate the superior performance of our algorithm in terms of both accuracy and speed compared to other state-of-the-art methods.
\end{abstract}


\section{Introduction}
\IEEEPARstart{L}ogistic regression (LR) is a special nonlinear regression model mainly used to solve classification problems, it has been particularly effective in some scenarios, such as neural networks \cite{zhang2017prediction,chizat2020implicit}, natural language processing \cite{wang2022deep}, bioinformatics \cite{tan2013minimax,wang2019lmtrda,7302571,min2018network}, and image classification \cite{krizhevsky2012imagenet,zheng2022hyperspectral}. 

However, up to the present date, the majority of logistic regression algorithms have assumed that their inputs are vectors. In fact, many datasets, such as electroencephalogram (EEG) and functional magnetic resonance imaging data \cite{wang2014clinical,wu2022bayesian}, multi-omics cancer data \cite{min2021tscca}, and other image datasets, take the form of high-dimensional tensors rather than vectors. The graphical illustration in Figure \ref{fig1} highlights the distinction between traditional logistic regression with vector inputs and multilinear logistic regression with tensor inputs. As depicted in the figure, the process of stretching matrices and tensors into vectors is a straightforward technique, but it may lead to the loss of relevant information of various dimensions \cite{tan2013logistic,guo2011tensor}. Moreover, expanding the data into an exceedingly high-dimensional vector could potentially trigger the curse of dimensionality \cite{wang2014clinical}. Consequently, there is an increasing interest in techniques that address two-dimensional (matrix-based) or higher-order (tensor-based) inputs. For example, it has been discovered that classic vector-based Principal Component Analysis (PCA) and Linear Discriminant Analysis (LDA) are less successful in face recognition tasks than two-dimensional PCA and LDA \cite{yang2004two}. And \cite{wang2014clinical} proposes the multilinear sparse logistic regression model with $\ell_1$, $\ell_2$-constrains and apply the block proximal gradient descent framework to solve multilinear logistic regression problem. 
\cite{wang2022hyperspectral} also proposes a feature learning phase that combines Khatri-Rao decomposition with multilinear logistic regression. However, a significant number of these algorithms often do not consider sparsity constraints, let alone addressing the more intricate problem of $\ell_{0}$-norm constraints.

And in the context of sparsity constraints, a notable strength of sparse logistic regression lies in its ability to perform feature selection \cite{guyon2002gene,wang2021extended}. Traditional logistic regression lacks an inherent mechanism for feature selection \cite{tan2013minimax}. In contrast, sparse logistic regression is able to select a subset of features with high predictive power and enhance generalization performance \cite{abramovich2018high,wang2019greedy}. 
In the realm of sparsity-constrained logistic regression, the work of \cite{plan2012robust} initially introduced the logistic regression model with a sparsity constraint. \cite{beck2013sparsity} delved into the statistical properties of the linear regression model subjected to a sparsity constraint. And regarding sparsity, one of the most effective and intuitive constraints is the $\ell_{0}$ norm. The $\ell_{0}$-norm is very suitable for sparse coding in machine learning, it can screen and remove the least ideal sparse feature components \cite{bradley1998feature,mairal2009online,min2021Group,min2022novel}. However, when the $\ell_{0}$-norm is employed as the sparsity criterion for optimization problem, it transforms into a NP-hard problem \cite{bolte2014proximal}. Therefore, the challenge of how to apply non-convex and non-smooth $\ell_{0}$-norms in logistic regression to ensure sparsity has attracted increasing attention. For example, \cite{pan2017convergent} employs the proximal operator to address the logistic regression with $\ell_{0}$-norm constraint, and \cite{wang2019greedy} extends this groundwork by incorporating a Newton iteration step. However, these approaches are tailored to conventional logistic regression and not applicable to multilinear sparse logistic regression. 

To address these issues, we propose a Multilinear Sparse Logistic Regression model with $\ell_0$-constraints ($\ell_0$-MLSR) for analyzing multi-way data (Figure \ref{fig1}(b)).
The $\ell_0$-norm constraints inherent in $\ell_0$-MLSR introduce nonconvexity, compounded by the nonconvex property arising from the multilinear operation, evidently, the $\ell_0$-MLSR model presents itself as a challenging non-convex and non-smooth optimization problem.
To address this issue, we devise a method known as Accelerated Proximal Alternating Linearized Minimization with Adaptive Momentum (APALM$^+$). Experimental results from synthetic and real-data datasets validate the effectiveness of APALM$^+$ when compared to existing algorithms. 

The main contributions of this paper can be summarized as follows.

(1) We present a Multilinear Sparse Logistic Regression model with $\ell_0$-constraints ($\ell_0$-MLSR), which allows us to take tensor data directly as input. Compared with $\ell_1$-norm and $\ell_2$-norm, $\ell_0$-norm constraint is more conducive to feature selection in the $\ell_0$-MLSR model.

(2) We propose an Accelerated Proximal Alternating Linearized Minimization method with Adaptive Momentum (APALM$^+$) to solve the $\ell_0$-MLSR model. We demonstrate that our method is globally convergent to a first-order critical point as well as establish convergence rate by using the Kurdyka–Łojasiewicz property. Additionally, we demonstrate the effectiveness of its adaptive extrapolation version. 

(3) We apply our algorithm to solve the $\ell_{0}$-MLSR model. The numerical experimental results on synthetic and real datasets show that our algorithm has superior numerical performance compared to state-of-the-art algorithms.

\begin{figure}[htp]
    \centering 
      \subfloat[Vector-based traditional logistic regression.]
    {
    \centering 
      \includegraphics[width=1\linewidth]{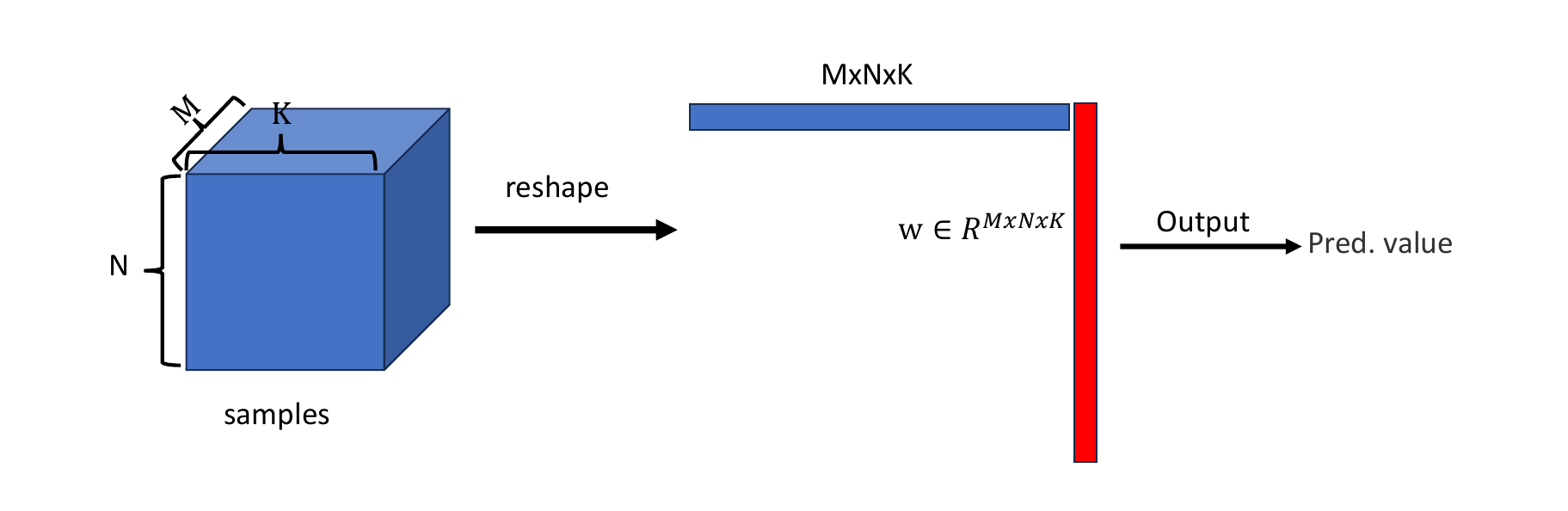}
    } \\
    \subfloat[Tensor-based multilinear logistic regression.]
    {
    \centering 
      \includegraphics[width=1\linewidth]{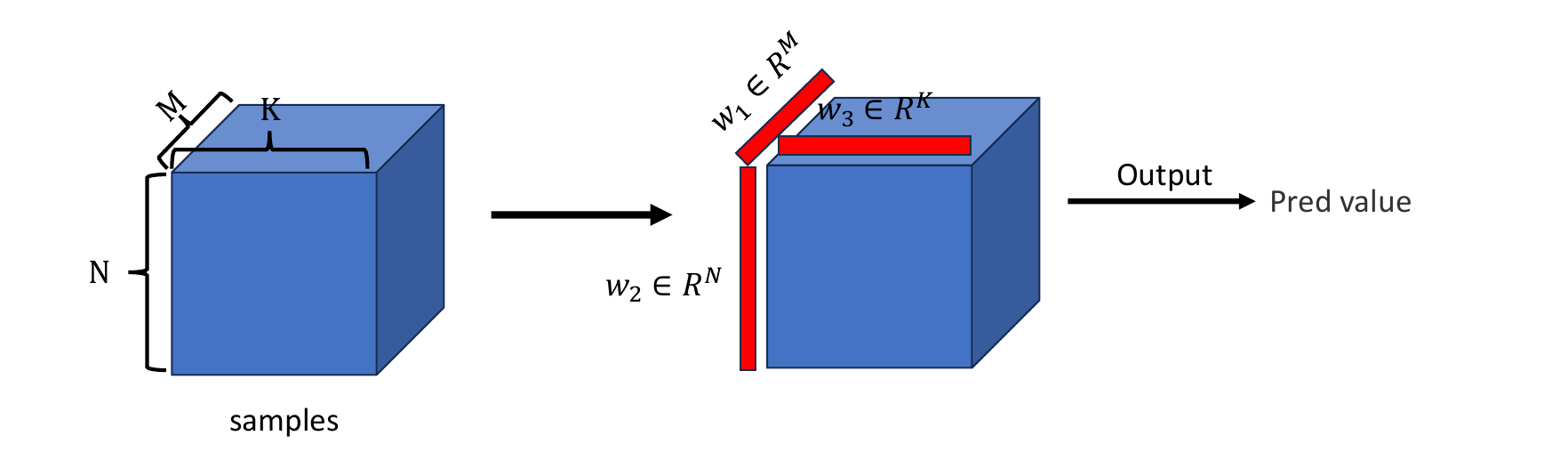}
    }
    \caption{Vector-based traditional logistic regression and multilinear logistic regression work on multidimensional data.
    }
    \label{fig1}
\end{figure}
\section{Symbol definitions and preliminaries}
\label{section: pre}
In this section, we will introduce some symbol definitions and the preliminary knowledge required in the derivation process. Table \ref{notation} summarizes the notation used in the paper. 

\begin{table}[!ht] 
\renewcommand\arraystretch{1.5}
\centering 
\caption{Notation}
\begin{tabular}{|p{3cm}|p{5cm}|} \hline 
Notation & Definition \\ \hline
$\left\{{x_{i}} \right\}^{n}_{i=1}$ & $(x_{1}, x_{2}. . . . . . , x_{n})$ \\
$\times_{k}$ & mode-k product \\
$\mathcal{W}$ &   $\left\{{w_{i}}\right\}^{p}_{i=1}$ \\

$(\mathcal{W}_{\backslash k },\Bar{w}_{k},\Bar{b})$ & $\left\{\left\{{w_{i}}\right\}^{k-1}_{i=1},\Bar{w_{k}},\left\{{w_{i}}\right\}^{n}_{i=k+1},\Bar{b}\right\}$ \\

$x_{(n)}$  & $\left\{{x_{i}} \right\}^{n}_{i=1}$ \\ 
$x_{(n)}-y_{(n)}$ &	$\left\{{x_{i}}-{y_{i}}\right\}^{n}_{i=1}$ \\ 
$L_{j}$ and $L_{\nabla H_{x_{j}}}$ & the Lipschitz constant denoting the variable of the $j$-th block of the function $H(x_{(n)})$ \\ 
$x^{k}_{i}$  & the $i$-th block of $x$ within the $k$-th outer loop \\ 
 
$\mathcal{C}^{p}_{L}(X)$ & a collection of functions satisfying the Lipschitz property \\

$\langle \centerdot, \centerdot \rangle$ & inner product \\
$||\centerdot||_{p}$ & $l_{p}$ norm \\
$||\left\{{x_{i}}\right\}^{n}_{i=1}||$ & $ \sum_{i=1}^{n}||(x_{i})|| \ ({\forall} x_{i} \in R^{n})$\\ \hline
\end{tabular}
\label{notation}
\end{table}
\subsection{Notation and preliminaries for nonconvex analysis}
\begin{definition}
If $A$ is a set, then defined the $A'$ is set of overall cluster points of $A$, $A^{o}$ is set of overall inter points of $A$. 
\label{d1}
\end{definition}
\begin{definition}
Proper Function: A function $g: \mathbb{R}^{n} \rightarrow (-\infty,+\infty]$ is said to be proper if $\mathrm{dom}$ $g \neq \emptyset$, where $\mathrm{dom} ~g =\left\{x \in \mathbb{R}:~ g(x)<\infty \right\}$. 
\end{definition}
\begin{definition}
Lower Semicontinuous Function: if $f$ is satisfied:
\begin{center}
$\lim_{k\rightarrow \infty} x_{k}=x, ~ f(x) \leq \operatorname*{lim}\operatorname*{inf}_{k\to\infty}f(x_{k}) ({\forall} x_{k} \in \mathrm{dom}~ f)$,
\end{center}
then $f$ is called lower semicontinuous at $\mathrm{dom}~ f$. 
\label{lower}
\end{definition}
\begin{definition}
Coercive Function: If $f$ is coercive, then $\left\{x | x \in R^{n}, ~ f(x) < a, ~{\forall} a \in \mathbb{R} \right\}$ is bounded and $\inf_{x} f(x) > -\infty$.
\label{d0}
\end{definition}
\begin{definition}
Let $f$ be a proper lower semicontinuous function, The Fréchet subdifferential of $f$ at $x$, written ${\hat{\partial}} f(x)$, is the set of all vectors $u$ which satisfy: 
\begin{center}
$\operatorname*{lim}_{y\neq x,y\to x}\cdot\frac{f(y)-f(x)-\langle u,\ y-x\rangle}{\|y-x\|}\ge0, $
\end{center}
when $ x \notin \mathrm{dom}~ f$, then set ${\hat{\partial}} f(x)=\emptyset$.
\label{d2}
\label{dsubdiff}
\end{definition}
\begin{definition}
The limiting subdifferential: $\partial f(x): =\{u\in\mathbb{R}^{n}: \exists x^{k}\to x,f(x^{k})\to f(x),u^{k}\to u,u^{k}\in\widehat{\partial}f(x^{k})\}.$
\end{definition}
\begin{proposition}
Fermat’s lemma: Let $f$ be a proper lower semicontinuous function. If $f$ has a local minimum at $x^{*}$, then $0 \in \partial f(x^{*})$. 
\label{p1}
\end{proposition}
\begin{proposition} \cite{bolte2014proximal}
Let f be a proper lower semicontinuous function, and $g$ be a continuously differential function. Then for any $x \in \mathrm{dom} ~f$, $\partial(f+g)(x) = \partial f (x) +\nabla g(x). $
\label{p2}
\end{proposition}
\begin{lemma}
Heine's\ theorem: $E$ is domain of function $f$, then ${\forall}x_{n}, ~ x _{0} \in E, ~ x_{0} \neq x_{n}, ~ \lim_{n \rightarrow \infty} x_{n}=x_{0}$, then: 
\begin{center}
$\lim_{x\rightarrow x_{0}}f(x)=f(x_{0}) \Leftrightarrow \lim_{n \rightarrow \infty}f(x_{n})=f(x_{0})$,
\end{center}
\label{tf4}
\end{lemma}
\begin{definition}
Set\ $f$: $X\rightarrow R, ~ X \subseteq R^{n}$, define the set $\mathcal{C}^{p}_{L}(X)$ as a set composed of all functions satisfying the following properties: 
   \begin{align}
    ||\nabla^{p} f(x)-\nabla^{p} f(y)|| \leq L*||x-y||. \notag
   \end{align}
$\mathcal{C}^{p}_{L}(X)$ is also known as a collection of functions satisfying the Lipschitz Property. in particular, when $p=1$, $\nabla_f$ is Lipschitz continuous. 
   \label{dlipchitz}
   \label{d3}
\end{definition} 
\begin{proposition}
Set\ $f$: $R^{n}\rightarrow R$, if $f \in \mathcal{C}^{1}_{L}(R^{n})$, then:
\begin{center}
 $f(y) \leq f(x)+\langle \nabla f(x),y-x \rangle+\frac{L}{2}||y-x||^{2}.$
 \end{center}
\label{p3}
\end{proposition}
\begin{definition}
Let $f$ be a proper lower semicontinuous function, $f$ is said to have the KŁ property on $\bar{u}\in d o m(\partial f) $, if there exists $\eta\in\left(0, +\infty\right]$, and $U$ is the neighborhood of $\bar{u}$, Any $u$ in $U$ that satisfies the condition $f({\bar{u}}) < f(u) < f({\bar{u}})+\eta$ has the following inequalities established: 
\begin{center}
$\phi^{\prime}(f(u)-f(\bar{u}))d i s t(0,\partial f(u))> 1$,
\end{center}
where $\phi$ is the desingularization function, i.e.: $\phi \in C^{1}([0,\eta)), ~ \phi '>0, ~ \phi(0)=0$, and $\phi$ is a concave function from in $(0,\eta)$. If $f $has the KŁ property at each point of $\mathrm{dom}$ $\partial f$, then $f$ is a KŁ function
\label{d8}
\end{definition}
Kurdyka–Łojasiewicz (KŁ) properties \cite{bolte2007clarke,bolte2014proximal,attouch2010proximal,attouch2013convergence} play a very important role for global convergence analysis in the non-convex optimization. 

For the convergence analysis in this paper, we adopt the following assumptions on the objective function family: 
\begin{assumption}
Set $J(x_{(n)})=H(x_{(n)})+\sum_{i=1}^{n} F_{i}(x_{i})$, where: 
\myspace{0}(i) $E_{H} \subseteq \prod_{i}^n \mathbb{R}^{d_{i}}$, $H: E_{H} \rightarrow\mathbb{R} $, $H \in \mathcal{C}^{1}_{L}(X)$ and $H$ is continuously differentiable. 
\myspace{0}(ii) $E_{F_{i}} \subseteq \prod_{i}^n \mathbb{R}^{d_{i}}$, $F_{i}: E_{F_{i}} \rightarrow \mathbb{R} $ are proper, lower semicontinuous functions. 
\myspace{0}(iii) The objective function $J$ satisfies the KŁ Property, and $J$ is a proper coercive function, and $J$ is lower bounded. 
\label{assump1}
\end{assumption}

\section{Proposed APALM$^{+}$ framework}
We introduce the details of $\ell_0$-MLSR (MLSR with $\ell_0$-norm) in this section, and propose our algorithm framework.
\subsection{Problem Statement}
Like traditional logistic regression, we assume each observation is a tensor $\mathcal{X}_{i} \in \mathbb{R}^{\prod_{i=1}^{p} d_{i}}$  and its response is $y_{i} \in \left\{0, 1\right\}$, then the predicted value is:
\begin{align}
 f_{(\mathcal{W},b)}(\mathcal{X}_{i})=\mathcal{X}_{i} ~\prod_{i=1}^{p} \times_{i} w_{i} +b,
 \label{e21}
\end{align}
where $w_{k} \in  \mathbb{R}^{1\times d_{k}}$, $b \in  \mathbb{R}$. Thus, the Multilinear Sparse Logistic Regression with $\ell_{0}$-norm ($\ell_0$-MLSR) definition we considered in this paper is given by:
\begin{align}
&& H(\mathcal{W},b)&=\sum_{i=1}^{n}\log(1+\exp({-y_{i} f(\mathcal{W},b)(\mathcal{X}_{i})})+\frac{\lambda}{2} ||\mathcal{W}||_{2}^{2}, \notag \\
&&s.t. ~||w_{i}||_{0} &\leq s_{i} ~({\forall} i \in \mathbb{N},~i\leq p ),
 \label{e22}
\end{align}
where $\lambda=\left\{\lambda_{i}\right\}_{i=1}^{p}$. 

We first give some characterizations of the objective function Eq. (\ref{e22}), and prove that it satisfies Assumption \ref{assump1}. 

The partial derivative of $H$ in Eq. (\ref{e22}) can be written as:
\begin{align}
&& &\nabla_{w_{i}}H(\mathcal{W},b)=\nabla_{w_{i}} H(\left\{{w_{j}}\right\}^{i-1}_{j=1},w_{i},\left\{{w_{j}} \right\}^{p}_{j=i+1},b) \notag \\
&& &=-\sum_{s=1}^{n} [1+\exp(-y_{s} f_{(\mathcal{W},b)} (\mathcal{X}_{s}) ]^{-1} \nabla_{w_{i}} f_{(\mathcal{W},b)}(\mathcal{X}_{s}) +\lambda_{i} w_{i},  \notag \\
&& &\nabla_{b} H(\mathcal{W},b)=-\sum_{s=1}^{n} [1+\exp(-y_{s} f_{(\mathcal{W},b)} (\mathcal{X}_{s}) ]^{-1} y_{s},
\label{e23}
\end{align}
where $\nabla_{w_{i}} f_{(\mathcal{W},b)}(\mathcal{X}_{s})=\mathcal{X}_{s}~\prod_{k=1,k\neq s}^{p} \times_{k} w_{k}$. Thus we have:
\begin{theorem}
The partial gradient $\nabla_{w_{i}}H(\mathcal{W},b)$ is Lipschitz continuous with constant ;
\begin{align}
\tau_{i}=\gamma* (\sqrt{2} \sum_{s=1}^{n}(||\nabla_{w_{i}} f_{(\mathcal{W},b)}(\mathcal{X}_{s})||_{2}+1)^2+\lambda_{i}) , \notag
\end{align}
where $\gamma > 1$.
\label{t0}
\end{theorem}
\begin{proof}
For any $(\mathcal{W}_{\backslash i },w_{i},b)$ and $(\mathcal{W}_{\backslash i },\Bar{w}_{i},\Bar{b})$, from \cite{wang2014clinical}, we know that:
\begin{align}
&& &\sum_{s=1}^{n} [1+\exp(-y_{s} f_{(\mathcal{W}_{\backslash i },w_{i},b)} (\mathcal{X}_{s}) ]^{-1} \nabla_{w_{i}} f_{(\mathcal{W},b)}(\mathcal{X}_{s}) \notag \\
&& &-\sum_{s=1}^{n} [1+\exp(-y_{s} f_{(\mathcal{W}_{\backslash i },\Bar{w}_{i},\Bar{b})} (\mathcal{X}_{s}) ]^{-1} \nabla_{w_{i}} f_{(\mathcal{W},b)}(\mathcal{X}_{s}) \notag \\
&& &\leq  \sqrt{2} \sum_{s=1}^{n}(||\nabla_{w_{i}} f_{(\mathcal{W},b)}(\mathcal{X}_{s})||_{2}+1)^2 ||(w_{i},b)- (\Bar{w}_{i},\Bar{b})||_{2}. \notag
\end{align}
then the triangle inequality gives us the following:
\begin{align}
&& &\frac{\nabla_{w_{i}}H(\mathcal{W}_{\backslash i },w_{i},b)-\nabla_{w_{i}}H(\mathcal{W}_{\backslash i },\Bar{w}_{i},\Bar{b})}{||(w_{i},b)- (\Bar{w}_{i},\Bar{b})||_{2}} \notag \\
&& &\leq  (\sqrt{2} \sum_{s=1}^{n}(||\nabla_{w_{i}} f_{(\mathcal{W},b)}(\mathcal{X}_{s})||_{2}+1)^2+\lambda_{i}).  \notag
\end{align}
when $i>p$, $\lambda_{i}=0$.
\end{proof}

\subsection{The Proposed APALM$^{+}$ Algorithm}
For Eq. (\ref{e22}), it can be rewritten by the penalty function method as: 
\begin{align}
&& &J(\mathcal{W},b)=H(\mathcal{W},b)+\sum_{i=1}^{p+1} F_{i}(w_{i}).
\label{e23}
\end{align}
where:
\begin{equation}
F_{i}(w_{i})=\delta_{i}(X_{i})=\left\{
\begin{aligned}
0&,\quad ||w_{i}||_{0} \leq s_{i},~ or ~ i > p,  \\
\infty&,\quad else.
\end{aligned}
\right. 
\notag
\end{equation}
We have proved that $H(\mathcal{W},b) \in \mathcal{C}^{1}_{L}(X) $ by Theorem \ref{t0}, and the fact that $F_{i}$ belongs to the lower semi-continuous (Definition \ref{lower}) function family is obvious \cite{bolte2014proximal}. And both $H(\mathcal{W},b)$ and $F_{i}$ satisfy the KŁ property \cite{wang2014clinical,bolte2014proximal}, thus Eq. (\ref{e23}) satisfies Assumption \ref{assump1}. 

Considering that the Eq. (\ref{e23}) is non-convex, we use the proximal operator to solve this problem. Let $S$ be the domain of function $J$ (Eq. (\ref{e23})), it can definited:  
\begin{align}
&& w^{k+1}_{j} &\in prox_{\frac{1}{\tau^{k}_{j}} F_{i}(w^{k}_{j})} (w_{j}^{k}-\frac{\nabla_{w^{k}_{j}}H(\mathcal{W}_{\backslash j },w^{k}_{j},b^{k})}{\tau_{j}^{k}}) \notag \\
&&&= \operatorname*{\arg\min} \limits_{w \in S^{d_{j}} \subseteq S} (F_{j}(w)+\frac{\tau^{k}_{j}}{2}||w-y^{k}_{j}||^{2} \notag \\
&& \ &+\langle \nabla_{w_{j}} H(\left\{{w_{i}^{k+1}} \right\}^{j-1}_{i=1}, y^{k}_{j}, \left\{{w_{i}^{k}} \right\}^{p}_{i=j+1}), w-y^{k}_{j} \rangle).
\label{prox} 
\end{align}
Therefore, it is easy to verify that the proximal operator Eq. (\ref{prox}) corresponding to $\ell_{0}$-MLSR is:
\begin{align}
 && w_{i}^{k+1} &\in prox_{\frac{1}{\tau^{k}_{i}} F_{i}(w^{k}_{i})}
 (U^{k}_{i}) \notag \\
 && &=\operatorname*{\arg\min} \limits_{w}\left\{\frac{\tau_{i}^{k}}{2}\left|\right|w-U^{k}_{i}\left|\right|_{2}^{2}+F_{i}(w^{k}_{i}) \right\}  \notag \\ 
 && &= \operatorname*{\arg\min} \limits_{w}\left\{\frac{\tau_{i}^{k}}{2}\left|\right|w-U^{k}_{i}\left|\right|_{2}^{2}:\left|\left|w\right|\right|_{0}\le s\right\},  \notag
\end{align}
where 
\begin{center}
$U^{k}_{i}=w_{i}^{k}-\frac{\nabla_{w^{k}_{i}}H(\mathcal{W}_{\backslash i },w^{k}_{i},b^{k})}{\tau_{i}^{k}},$
\end{center}
and there is no need to employ the proximal operator for the bias parameter $b$ because it is unrestricted. 

To this end, we propose an Accelerated proximal alternating minimal linearization with adaptive momentum (APALM$^+$) for solving Eq.(\ref{e23}) (see Algorithm \ref{APALM$^+$}).

\begin{algorithm}[!h]
\caption{\small APALM$^+$: Accelerated proximal alternating minimal linearization with adaptive momentum for MSLR}\label{APALM$^+$}
    \SetAlgoLined
     \LinesNumbered
        \KwIn{$\left\{{w^{1}_{i}} \right\}^{p}_{i=1}=\left\{{w^{0}_{i}} \right\}^{p}_{i=1} \in \mathrm{dom} \ J$, $b^{1}=b^{0} \in \mathbb{R}$,   $ k_{\max}=c, ~ t\in (1, \infty) , \beta_{\max} \in [0, 1),~ \beta_{1} \in [0, \beta_{\max}]$}
        
        \KwOut{$\left\{{w^{k+1}_{i}} \right\}^{p}_{i=1}$, $b^{k+1}$}
        
        \For {$k = 1$ to $k_{\max}$}{
        $\left\{{y^{k}_{i}} \right\}^{p}_{i=1}$=$\left\{{w^{k}_{i}} \right\}^{p}_{i=1}$+$\beta_{k} (\left\{{w^{k}_{i}} \right\}^{p}_{i=1}-\left\{{w^{k-1}_{i}} \right\}^{p}_{i=1})$.
        
        $y^{k}_{p+1}=b^{k}+\beta_{k}(b^{k}-b^{k-1})$.
        
        \eIf{\begin{equation}\label{if}J(\left\{{y^{k}_{i}} \right\}^{p+1}_{i=1}) \leq J(\left\{{w^{k}_{i}} \right\}^{p}_{i=1},b^{k})\end{equation}} 
        {$\left\{{w^{k}_{i}} \right\}^{p}_{i=1}$=$\left\{{y^{k}_{i}} \right\}^{p}_{i=1}$, $b^{k}=y^{k}_{p+1}$,  
        
        $\beta_{k+1}=\min(\beta_{max},t*\beta_{k})$.}{

            $\beta_{k+1}=\frac{\beta_{k}}{t}$.
        }
        \For {$i = 1$ to p}{
         
         $ w_{i}^{k+1} \in prox_{\frac{1}{\tau^{k}_{i}} F_{i}(w^{k}_{i})} (w_{i}^{k}-\frac{\nabla_{w^{k}_{i}}H(\mathcal{W}_{\backslash i },w^{k}_{i},b^{k})}{\tau_{i}^{k}}) $
        } 
        $b^{k+1}=b^{k}-\frac{\nabla_{b}H(\mathcal{W}^{k+1},b^{k})} {\tau_{i}^{k}}$
        }
        
\end{algorithm}
\begin{remark}
(\romannumeral1) For the step 6 and 8 of Algorithm \ref{APALM$^+$}, APALM$^+$ is still successful if the extrapolation parameters do not follow an adaptive strategy, such as $t_{k+1}=\frac{1+\sqrt(1+4*t_{k}^2)}{2}, ~ \beta_{k+1}=\frac{t_{k+1}-1}{t_{k}}$. This extrapolated parameter update strategy comes from \cite{beck2009fast}.

(\romannumeral2) When the extrapolated parameters are not adaptively updated, we refer to APALM$^+$ as APALM.
\end{remark}

Compared with the existing PALM algorithm \cite{xu2017globally,pock2016inertial,le2020inertial,min2023structured}, our algorithm has the property of adaptive momentum acceleration. We can also rigorously demonstrate in the next part and related material that our algorithm can not only make the objective function Eq. (\ref{e23}) monotonously decreasing and convergent, it also has global convergence. 

\section{Convergence Analysis}
\label{Convergence}
In this section, we will demonstrate the convergence of the proposed APALM$^+$ (Algorithm \ref{APALM$^+$}). 
To this end, let $z^{k}=\left\{\left\{w^{k}_{i} \right\}^{p}_{i=1},b^{k}\right\},~c^{k}=\left\{y^{k}_{i} \right\}^{p+1}_{i=1}$, we will prove that Algorithm \ref{APALM$^+$} satisfies the following three properties in Property \ref{based} when it is used to solve problem (\ref{e23}): 
\begin{Property}   
(\romannumeral1) Sufficient decrease property, i.e., there exists a positive constant $\rho$ such that $J(\mathcal{W}^{k+1},b^{k+1})\leq
    J(\mathcal{W}^{k},b^{k}) -\rho||z^{k+1}-c^{k}||$.

(\romannumeral2) A subgradient lower bound for the iterates gap, i.e., there exists a positive constant $\rho_{2}$ such that $\rho_{2}||z^{k+1}-c^{k}||\geq ||g^{k}||$ where $g^{k} \in \partial J(\mathcal{W}^{k+1},b^{k+1})$.

(\romannumeral3) Using the Kurdyka–Łojasiewicz (KŁ) Property (\cite{bolte2014proximal}), 
the generated sequence by Algorithm \ref{APALM$^+$} is a Cauchy sequence.
\label{based}
\end{Property}

\subsection{Sufficient decrease of the objective function}
First, we will prove the Property \ref{based} (\romannumeral1).
\begin{theorem}
suppose that Assumption \ref{assump1} hold, let $x_{(N)}$ are sequences generated by Algorithm \ref{APALM$^+$}. \\ 
(\romannumeral1):
$J(\mathcal{W}^{k+1},b^{k+1})$ is nonincreasing and in particular:
\begin{align}
    J(\mathcal{W}^{k+1},b^{k+1})\leq
    J(\mathcal{W}^{k},b^{k}) -\rho||z^{k+1}-c^{k}||, \notag
\end{align}
where $\rho>0$, and $\rho$ is a positive constant.  \\
(\romannumeral2): we have: 
\begin{center}
$\lim_{k \to \infty} ||z^{k+1}-c^{k}||=0$.
\end{center}
\label{t1}
\end{theorem}
\begin{proof}
See Appendix \ref{TP1}.
\end{proof}

\subsection{Subgradient lower bound for the iterates gap}
We will prove the Property \ref{based} (\romannumeral2), i.e., the derivative set of Eq. (\ref{e23}) is the critical point set. 
\begin{theorem}
If we define
\begin{center}
$g_{w_{j}^{k+1}}=\nabla_{w_{j}}H (\mathcal{W}_{\backslash i },w^{k+1}_{j},b^{k}))-\nabla_{w_{j}}H (\mathcal{W}_{\backslash i },y^{k}_{j},y^{k}_{p+1})+\frac{1}{\tau^{k}_{j}
}||y^{k}_{j}-w^{k+1}_{j}||$.
\end{center}
Then we have:
\begin{equation}
\begin{aligned}
||\left\{g_{w^{k+1}_{i}} \right\}^{p}_{i=1}|| \leq \rho_{b}||z^{k+1}-c^{k}||,
\end{aligned}
\label{e35}
\end{equation}    
where $\rho_{b}>0$.
\label{t2}
\end{theorem}
\begin{proof}
See Appendix \ref{TP2}.
\end{proof}

From Definition \ref{d1}, define $z'$ is the derivative set of $\left\{z\right\}$, thus: 

\begin{theorem}
Let $z^{k}$ be a sequence generated by Algorithm \ref{APALM$^+$}. Then $J$ is a constant on $z'$ and $z' \subseteq crit J$.
\label{c5}
\end{theorem}
\begin{proof}
See Appendix \ref{TP3}.
\end{proof}

\subsection{Global convergence under KŁ Property}
We will prove the Property \ref{based} (\romannumeral3), i.e., the sequence generated by Algorithm \ref{APALM$^+$} has global convergence (see Theorem \ref{glo}). 

\begin{theorem}
The\ sequence\ $\left\{z^{k}\right\}$ generated by Algorithm \ref{APALM$^+$} is converged when $\beta \in [0, \beta_{\max}]$, i.e: $\lim_{k\rightarrow \infty} ||z^{k+s}-z^{k}||=0(\forall s\in \mathbb{N})$
\label{glo}
\end{theorem}
\begin{proof}
See Appendix \ref{TP4}.
\end{proof}

Under KŁ inequality, we can obtain convergence rate results as following: 
\begin{theorem} 
(Convergence rate): Let Assumption \ref{assump1} is true, and let $z^{k}$ be a sequence generated by Algorithm \ref{APALM$^+$}, the desingularizing function has the form of $\phi(t) = \frac{\theta}{C}t^{\theta}$, with $\theta\in (0, 1], ~ c > 0$. Let $J^{*}=J(e)(e \in z'), ~ r^{k} = J(z^{k})-J^{*}$. The following assertions hold: \\
(\romannumeral1): If $\theta =1$, the Algorithm \ref{APALM$^+$} terminates in finite steps. \\
(\romannumeral2): If $\theta \in [\frac{1}{2}, 1)$, then there exist a integer $k_{2}$ such that
\begin{center}
$r^{k}\leq (\frac{d_{1}C^{2}}{1+d_{1}C^{2}}),{\forall k_{2} \geq k}$.
\end{center}
(\romannumeral3): If $\theta \in (0, \frac{1}{2})$, then there exist a integer $k_{3}$  such that
\begin{center}
$r^{k}\leq [\frac{C}{(k-k_{3})d_{2}(1-2\theta)}]^{\frac{1}{1-2\theta}},{\forall k_{3} \geq k}$.
\end{center}
where 
\begin{center}
$d_{1}=(\frac{\rho_{b}}{\rho})$, $d_{2}=\min\left\{\frac{1}{2d_{1}C}, \frac{C}{1-2\theta}(2^{\frac{2\theta-1}{2\theta-2}})r_{0}^{2\theta-1}\right\}$.
\end{center}
\label{rate}
\end{theorem}
\begin{proof}
See Appendix \ref{TP5}.
\end{proof}

\section{Numerical experiments} 
\label{section: B}
In this section, we will present the experimental results to evaluate the effectiveness of the proposed method,including synthetic and real-data examples. All algorithms run on this configuration: 12th Gen Intel(R) Core(TM) i7-12700 2.10 GHz, RAM 32.0 GB (31.8 GB available); 64-bit operating system; realized on the configuration of Matlab 2021b. The fundamental tensor computation was based on Tensor Toolbox 3.5 \cite{bader2006algorithm}. The code is available at \url{https://github.com/Weifeng-Yang/MLSR}.

\subsection{Proposed Algorithms and Baseline Algorithms}
\label{subsection: asup}
Based on our proposed APALM$^+$ (See Algorithm \ref{APALM$^+$}),  we propose two transformation forms: 

APALM:  It is an accelerated block proximal algorithm with non-adaptive momentum ;

APALM$^+$:  It is an accelerated block proximal algorithm with adaptive momentum;

We compare these algorithms with state-of-the-art algorithms for solving logistic regression: 
\begin{itemize}

    \item [1)] GIST  \cite{gong2013general}:  General iterative shrinkage and thresholding algorithm is used to solve traditional logistic regression with $\ell_{1}$-norm or $\ell_{2}$-norm. 

    \item [2)] GPGN  \cite{wang2019greedy}:  Greedy Projected Gradient-Newton Method is used to solve the traditional sprase logistic regression by using proximal gradient method and newton steps. 

   \item [3)] BPGD \cite{wang2014clinical}:  Block Proximal Gradient Descent for Multilinear Sparse Logistic Regression is a first-order method that used to solve multilinear logistic regression with $\ell_{1}$-norm or $\ell_{2}$-norm. 

    \item [4)] IBPG  \cite{le2020inertial}: Inertial block proximal gradient (IBPG) method is a first-order accelerated algorithm. By adding a condition, such as strong convex, IBPG allows repeated updates, but it is not a monotonically decreasing method. 

    \item [5)] BPL  \cite{xu2017globally}:  Randomized/deterministic block prox-linear (BPL) method is a BCD-based first-order method for the nonconvex and nonsmooth problems. 
    
\end{itemize}
we give the parameter update strategy of the extrapolated parameters of the algorithms in Table \ref{betaupdate} (applicable to the following all examples). 
\begin{table*}
\renewcommand\arraystretch{1.4}
\centering
\caption{Update strategy of the extrapolated parameters $\beta_{k}$ and the update condition of APALM, APALM$^+$}\label{betaupdate}
\scalebox{1}{
\begin{tabular*}{\linewidth}{l|l|l} 
\toprule     
Algorithm & Update strategy of extrapolated parameters  & Remark \\  
\midrule   
APALM & \thead[l]{$t_{k+1}=\frac{1+\sqrt(1+4*t_{k}^{2})}{2}, ~ \beta_{k+1}=\frac{t_{k}-1}{t_{k+1}}$}  & $t_{1}=1$; This strategy comes from  \cite{beck2009fast} on solving convex.\\  \hline    

APALM$^+$ & \thead[l]{$\beta_{k+1}=\left\{
    \begin{aligned}
    && & min(\beta_{max}, t\beta_{k}), ~if~Eq.~\ref{if} \ is \ true\\
    && & \frac{\beta_{k}}{t}, ~else \\
    \end{aligned}
    \right. $}& $\beta_{max}<1, ~t>1$. \\ 
\bottomrule      
\end{tabular*}
}
\end{table*}

\subsection{Experiments on Synthetic Data}
We first constructed some synthetic data sets to assess the presence of the following properties of the Algorithm \ref{APALM$^+$}:

(\romannumeral1): Whether Algorithm \ref{APALM$^+$} can effectively discover the potential structure of tensor data.

(\romannumeral2): Scalability of Algorithm \ref{APALM$^+$} when implemented on datasets of different sizes.

(\romannumeral3): Whether the Algorithm \ref{APALM$^+$} input in the form of tensor can effectively solve the problem of curse of dimensionality in traditional logistic regression.

(\romannumeral4): Whether the Algorithm \ref{APALM$^+$} with adaptive extrapolation parameters is more efficient.

Therefore, following \cite{wang2014clinical}, we constructed a dataset with two distinct modes to answer the above questions. For the first mode of synthesize dataset, we construct a square matrix 200x200 size for each data object, each category has 500 samples. For the second mode of the synthetic data set, we decrease the number of samples for each category to 100, but increase the data dimension to 800x800. The elements in the data matrices are generated independently of $\mathcal{N}(0, 1)$—a univariate Gaussian distribution with zero mean and unit variance, and the upper-left 20 × 20 block was different for the data matrices in class 1 and class 0. We generate two vectors $v_{1} \in \mathbb{R}^{20}$ and $v_{2} \in \mathbb{R}^{20}$ whose elements are generated independently of uniform distribution between 0 and 1. For any data matrix X from class 0, we set it to satisfy the following expression:
\begin{center}
$v_{1}^{T}*\Bar{X}*v_{2}+1\geq 0.5$,
\end{center}
where $\Bar{X}$ is the 20x20 block in the upper left corner of matrix X. And for any data matrix Y from class 1; 
\begin{center}
$v_{1}^{T}*\Bar{Y}*v_{2}+1\leq -0.5$,
\end{center}
where $\Bar{Y}$ is the 20x20 block in the upper left corner of matrix Y. It is clear that these data have a unique two-dimensional correlation structure. 

Next, we explain how to choose the parameters for Algorithm \ref{APALM$^+$}, we set each initial hyperparameter according to the initial conditions as: $t=1.3, ~ \beta_{1}=0.6, ~ \beta_{max}=0.9999, ~ \gamma=1. 5, \lambda=\left\{2*10^{-4}\right\}_{i=1}^{2}$, and we select 80$\%$ as training set and 20$\%$ as test set. We run the algorithms for a specific sparsity setting: the number of non-zero elements in each vector cannot exceed $30\%$ of the total number of elements, The initial point of $(\mathcal{W},b)$ is set to a standard Gaussian distributed random sparse definite vector.

In order to show the effectiveness of the proposed algorithms and compare the accuracy of the solutions, we will evaluate them from two aspects:
(1) we run all algorithms ten times with the same hyperparameters, each time using the different random initialization point, and let all algorithms run for the same length of time. For the first mode of synthesize dataset, the running time was set by 200 seconds, For the second mode of synthesize dataset, the running time was set by 400 seconds. The relationship between the predicted AUC (area under the receiver operating characteristic curve) and the objective function over time is shown in Figure \ref{figsythesize}; 
(2) For the second aspect, using the following stopping criteria: 
\begin{align}
\frac{|J^{k+1}-J^{k}|}{n}<1e-5, ~or \notag \\
\frac{||\nabla J^{k+1}-\nabla J^{k}||_{2}}{n}<1e-4.
\label{stop}
\end{align}
where $n$ represents the sample number. For the first mode of synthesize dataset, the maximum running time was set by 400 seconds. For the second mode of synthesize dataset, the maximum running time was set by 600 seconds. We also run all algorithms ten times with the same hyperparameters and use the different random initialization points each time. The test results are presented in Table \ref{synthesize}. 


\begin{figure}[hbpt]
    \centering 
      \subfloat[First mode: 1000x200x200]
    {
    \centering 
      \includegraphics[width=1\linewidth]{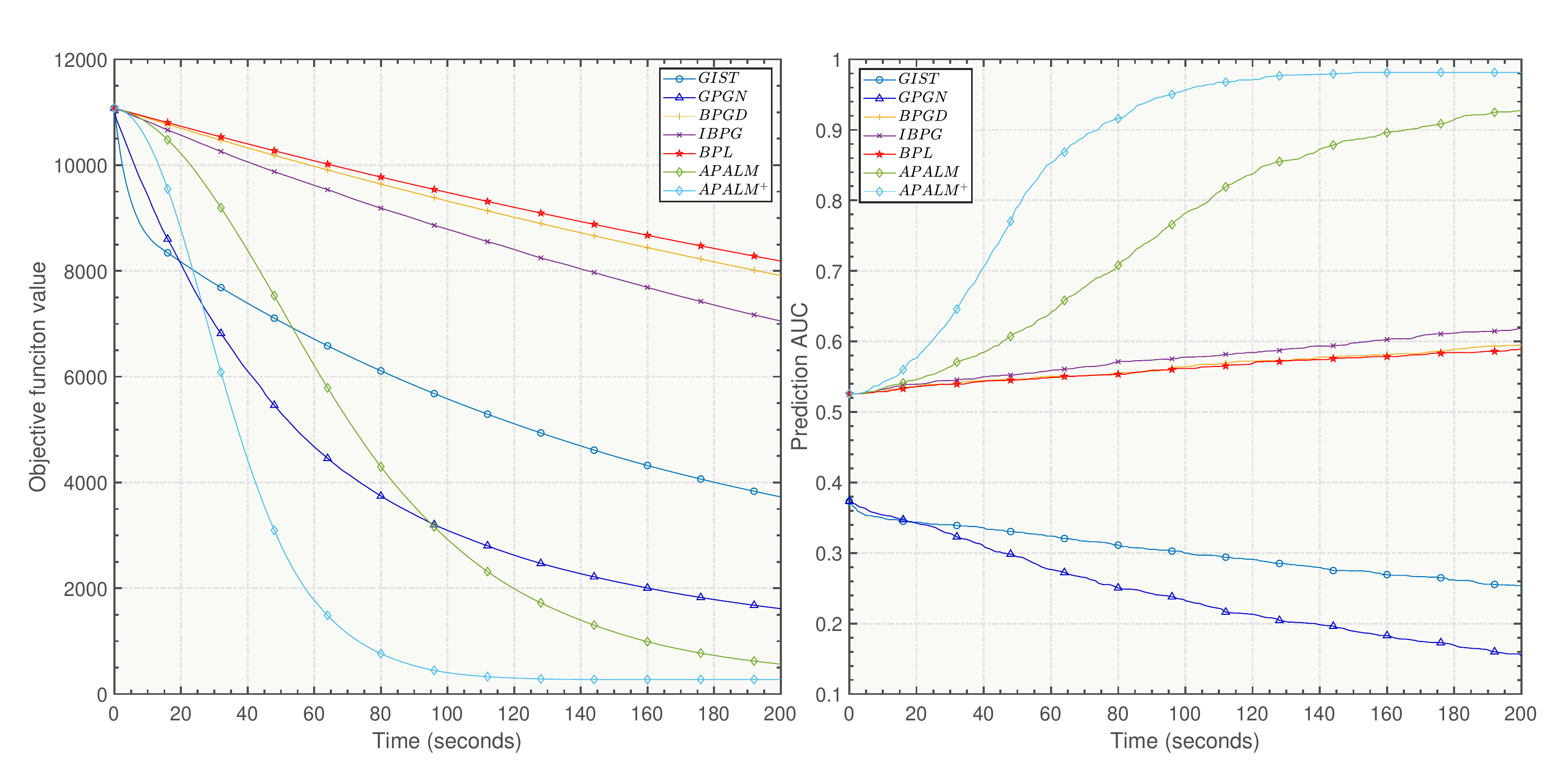}
    } \\
    \subfloat[Second mode: 200x800x800]
    {
    \centering 
      \includegraphics[width=1\linewidth]{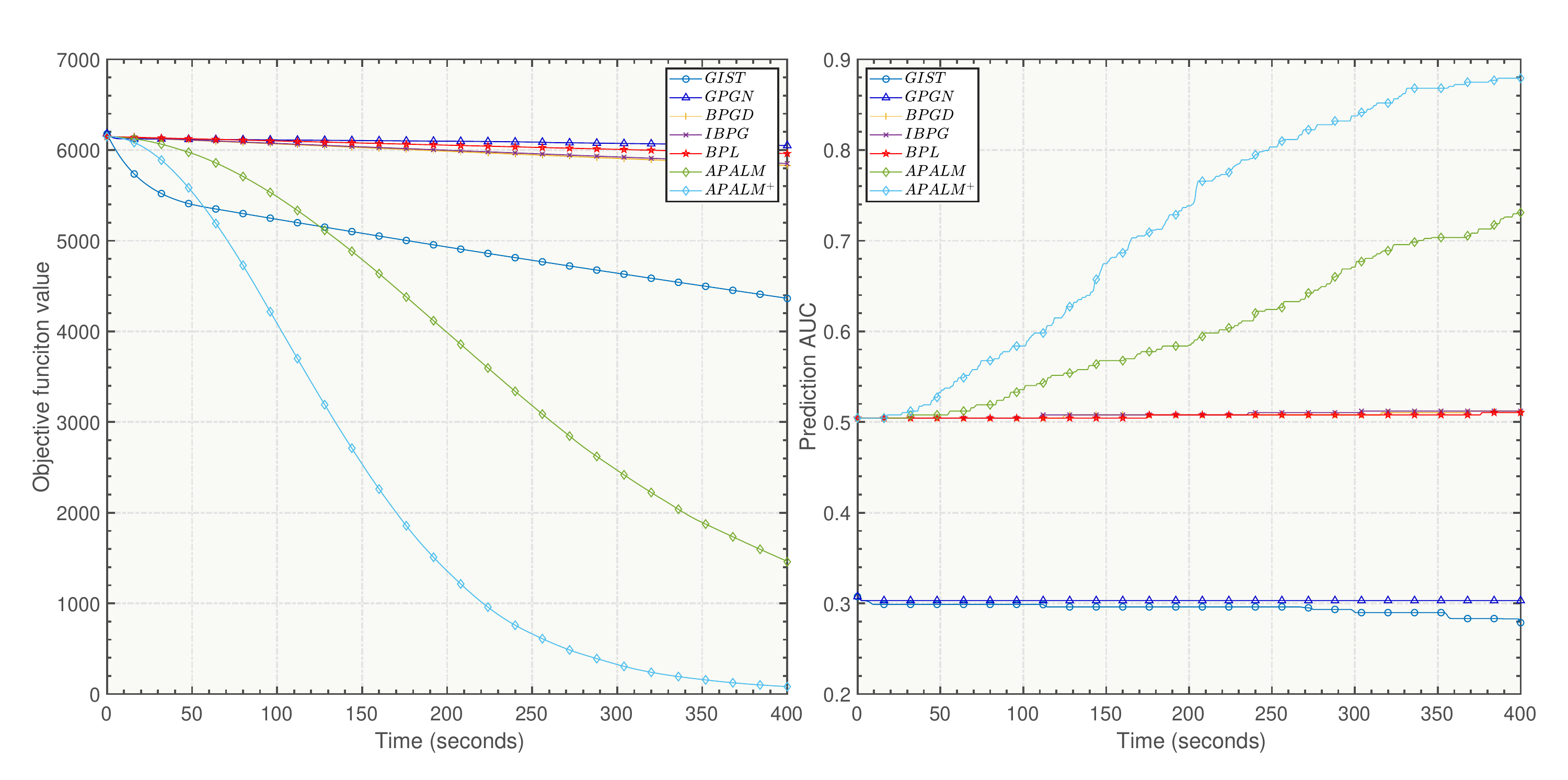}
    }
    \caption{
    Comparison of the average convergence speed of the objective function and AUC  of different algorithms on synthetic dataset. The convergence curves show the average value of the objective function and the average AUC. 
    }
    \label{figsythesize}
\end{figure}

\begin{table*}
\renewcommand\arraystretch{1.4}
\centering
\caption{Comparison of algorithms applied on the real datasets. The results are the average and standard deviation of ten runs, and the bold indicates the best numerical performance under the set termination conditions.}
\begin{tabular*}{\linewidth}{p{2cm}|p{1.7cm}|p{3.15cm}|p{2.8cm}|p{3cm}|p{2.9cm}} 
\toprule     
Data & Algorithm  &Objective function value & Time (seconds) & Accuracy (\%) $\pm$ std & AUC $\pm$ std \\
\midrule   
\multirow{7}{*}{First mode: } & \multirow{1}{*}{GIST}  & $2.33*10^{3}$ & 400 & $19.80\pm6.18$ & $0.20\pm0.07$  \\ \cline{2-6}

\multirow{7}{*}{1000x200x200} & \multirow{1}{*}{GPGN} & $763.16$ & 400 & $9.25\pm2.35$ & $0.09\pm0.02$  \\  \cline{2-6}

\multirow{7}{*}{} & BPGD &  $6.58*10^{3}$ 
& 400 & $62.75\pm19.86$ &  $0.63\pm0.20$\\ \cline{2-6} 

\multirow{7}{*}{} & IBPG  & $5.20*10^{3}$ & 400 & $66.25\pm19.40 $  & $0.66\pm0.19$  \\ \cline{2-6}

\multirow{7}{*}{} & BPL  & $7.07*10^{3}$ & 400 & $61.80\pm19.97$ &$0.61\pm0.19$   \\ \cline{2-6}

\multirow{7}{*}{} & \multirow{1}{*}{APALM}  & $293.49$ & 366.56 & $92.85\pm5.23$ & $0.92\pm0.05$  \\ \cline{2-6}

\multirow{7}{*}{} & \multirow{1}{*}{APALM$^+$}  & $\bold{275.77}$ & $\bold{100.10}$ & $\bold{94.85\pm5.15}$ & $\bold{0.95\pm0.05}$ \\ \hline 

\multirow{7}{*}{Second mode: } & \multirow{1}{*}{GIST}  & $3.51*10^{3}$ & 600 & $26.50\pm6.05$ & $0.28\pm0.075$  \\ \cline{2-6}

\multirow{7}{*}{200x800x800} & \multirow{1}{*}{GPGN} & $5.22*10^{3}$ & 600 & $28.50\pm6.63$ & $0.29\pm0.079$  \\  \cline{2-6}

\multirow{7}{*}{} & BPGD &  $4.95*10^{3}$ 
& 600 & $57.00\pm9.92$ &  $0.59\pm0.11$ \\ \cline{2-6} 

\multirow{7}{*}{} & IBPG  & $4.99*10^{3}$ & 600 & $57.25\pm10.02$   & $0.59\pm0.11$  \\ \cline{2-6}

\multirow{7}{*}{} & BPL  & $5.11*10^{3}$ & 600 & $56.75\pm10.00$ & $0.58\pm0.11$   \\ \cline{2-6}

\multirow{7}{*}{} & \multirow{1}{*}{APALM}  & $494.34$ & 556.18 & $81.00\pm7.92$ & $0.83\pm0.10$  \\ \cline{2-6}

\multirow{7}{*}{} & \multirow{1}{*}{APALM$^+$}  & $\bold{57.10}$ & $\bold{372.85}$ & $\bold{87.00\pm3.67}$ & $\bold{0.88\pm0.06}$ \\ \hline 
\end{tabular*}
\label{synthesize}
\end{table*}

Utilizing Table \ref{synthesize} and Figure \ref{figsythesize} as references, our observations include 
(i) APALM$^+$ demonstrates a notable ability to uncover the underlying data structure within the matrix compared to other algorithms. In contrast, traditional sparse logistic regression methods (GPGN and GIST) exhibit limited effectiveness in this aspect; 
(ii) Compared with other algorithms, APALM$^+$ consistently surpasses them across all evaluation. 
(iii) The impact of APALM$^+$ (adaptive momentum) is better than APALM (non-adaptive); 
(iv) As the sample dimension of the synthetic dataset increases, the iteration speed of GPGN and GIST notably slows down. However, algorithms employing tensor-based input, like APALM$^+$, remain largely unaffected. This illustrates the effectiveness of tensor-based logistic regression in mitigating the curse of dimensionality. 

\subsection{Experiments on Real Data}
We consider three real data sets: Concrete Crack Images for Classification\footnote{\url{https://data.mendeley.com/datasets/5y9wdsg2zt/1}} (227*227*3) \cite{zhang2016road}, GochiUsa-Faces\footnote{\url{https://www.kaggle.com/datasets/rignak/gochiusa-faces}} (26*26*3$\sim$987x987*3) \cite{GochiUsa_Faces}, and Br35H :: Brain Tumor Detection 2020\footnote{\url{https://www.kaggle.com/datasets/ahmedhamada0/brain-tumor-detection}} (201*251$\sim$1024*1024) \cite{br35h-::-brain-tumor-detection-2020_dataset}. For Concrete Crack Images, we will take 500 samples from each category respectively to utilize. For the GochiUsa-Faces dataset, we take 500 samples from the categories ``Chino" and ``Chiya" respectively under the folder ``DANBOORU" to utilize and reshape each sample to 128*128*3 size. For Br35H :: Brain Tumor Detection 2020, we will take 500 samples from each category respectively to utilize and reshape each sample to 256*256 size. For each sample, scaling to $[-1,1]$ in accordance with the feature is the normalization and standardization approach used.

We first explain how to choose the parameters for Algorithm \ref{APALM$^+$}, we set each initial hyperparameter according to the initial conditions as: $t=1.3, ~ \beta_{1}=0.6, ~ \beta_{max}=0.9999, ~ \gamma=1. 5, \lambda=\left\{2*10^{-4},2*10^{-4},2*10^{-3}\right\}$, and we select 80$\%$ as training set and 20$\%$ as test set. We run the algorithms for a specific sparsity setting: the number of non-zero elements in each vector cannot exceed $30\%$ of the total number of elements, The initial point of $(\mathcal{W},b)$ is set to a standard Gaussian distributed random sparse definite vector.

In order to show the effectiveness of the proposed algorithms and compare the accuracy of the solutions, we will evaluate them from two aspects: 
(1) we run all algorithms ten times with the same hyperparameters, each time using the different random initialization point, and let all algorithms run for the same length of time. For the Concrete Crack Image and GochiUsa-Faces datasets, the running time was set by 300 seconds. For brain Tumor Detection 2020, the running time was set by 1300 seconds. The relationship between the predicted AUC and the objective function over time is shown in Figure \ref{figreal};  
(2) For the second aspect, using the same stopping criteria as Eq. (\ref{stop}), the maximum running time was set by 600 seconds for Concrete Crack Images and GochiUsa-Faces datasets.  For Brain Tumor Detection 2020, the maximum running time was set by 2000 seconds. We also run all algorithms ten times with the same hyperparameters and use the different random initialization points each time. The test results are presented in Table \ref{real}.

\begin{figure}[hbpt]
    \centering 
      \subfloat[Dataset: Concrete Crack Images for Classification]
    {
        \centering 
      \includegraphics[width=1\linewidth]{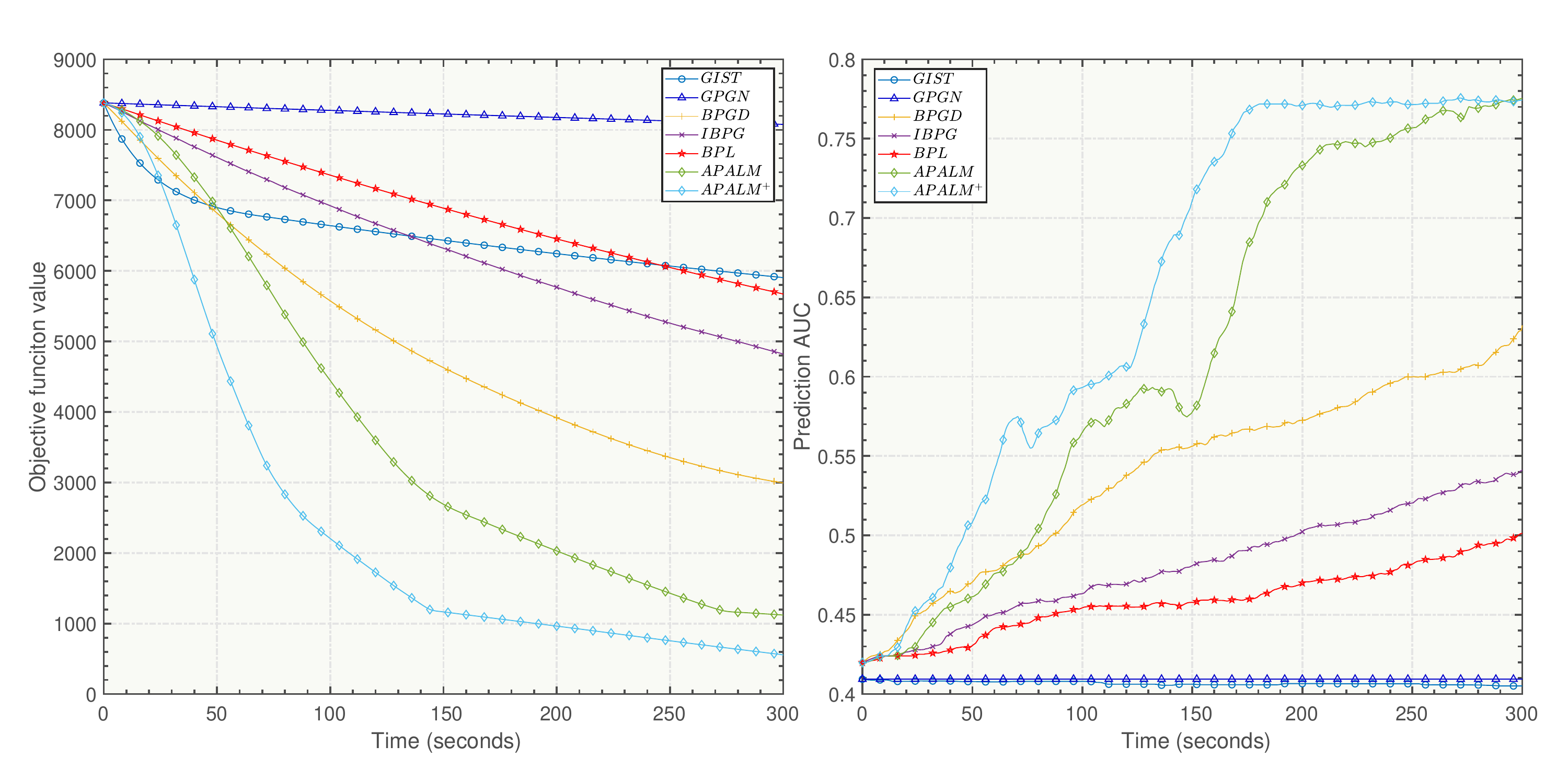}
    } \\
    \subfloat[Dataset: GochiUsa-Faces]
    {
      \includegraphics[width=1\linewidth]{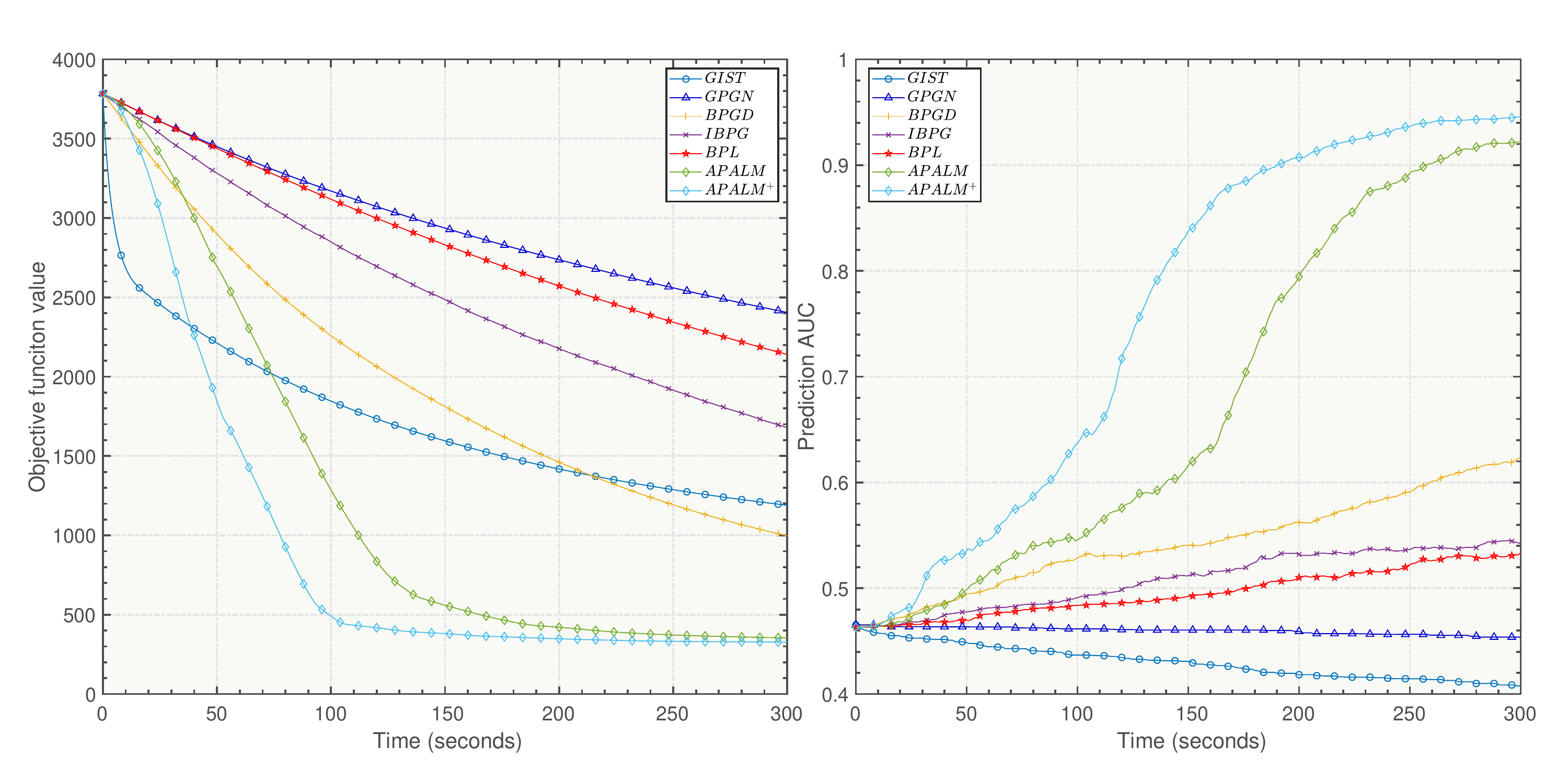}
    }\\
    \subfloat[Dataset: Brain Tumor Detection 2020]
    {
      \includegraphics[width=1\linewidth]{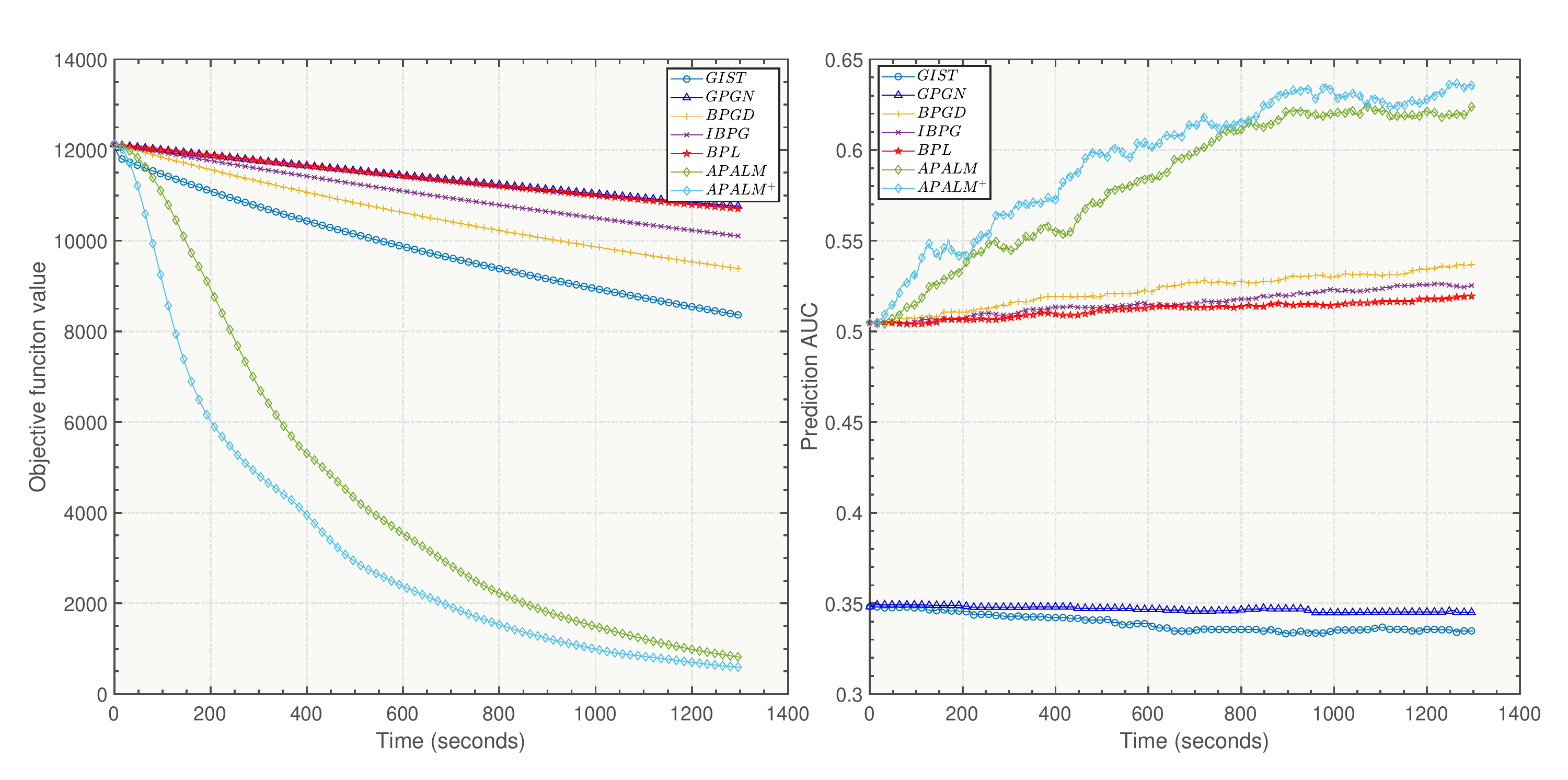}
    }
    \caption{
    Comparison of the average convergence speed of the objective function and AUC of different algorithms on real dataset. The Convergence curves shows the average objective function value and average AUC. 
    }
    \label{figreal}
\end{figure}

\begin{table*}
\renewcommand\arraystretch{1.4}
\centering
\caption{Comparison of algorithms applied on the real datasets. The results are the average and standard deviation of ten runs, and the bold indicates the best numerical performance under the set termination conditions.}
\begin{tabular*}{\linewidth}{p{2cm}|p{1.7cm}|p{3.15cm}|p{2.8cm}|p{3cm}|p{2.9cm}}  
\toprule     
Data & Algorithm  &Objective function value & Time (seconds) & Accuracy (\%) $\pm$ std & AUC $\pm$ std \\
\midrule   
\multirow{7}{*}{Concrete Crack} & \multirow{1}{*}{GIST}  & $4.05*10^{3}$ & 600 & $40.65\pm11.62$ & $0.39\pm0.11$\\ \cline{2-6}

\multirow{7}{*}{Images for} & \multirow{1}{*}{GPGN} & $5.72*10^{3}$ & 600 & $42.20\pm12.05$ & $0.40\pm0.11$ \\  \cline{2-6}

\multirow{7}{*}{Classification} & BPGD &  $1.34*10^{3}$ 
& 600 & $72.25\pm11.11$ &  $0.72\pm0.12$\\ \cline{2-6} 

\multirow{7}{*}{} & IBPG  & $2.34*10^{3}$  & 581.06 & $67.3\pm19.24$   & $0.67\pm0.19$  \\ \cline{2-6}

\multirow{7}{*}{} & BPL  & $3.24*10^{3}$ & 587.23 & $64.85\pm18.36$ & $0.65\pm0.18$  \\ \cline{2-6}

\multirow{7}{*}{} & \multirow{1}{*}{APALM}  & $391.15$ & 519.02 & $82.55\pm5.82$ & $0.83\pm0.062$  \\ \cline{2-6}

\multirow{7}{*}{} & \multirow{1}{*}{APALM$^+$}  & $\bold{389.11}$ & $\bold{399.39}$ & $\bold{83.3\pm3.37}$ & $\bold{0.84\pm0.034}$ \\ \hline

\multirow{7}{*}{GochiUsa-Faces } & \multirow{1}{*}{GIST}  & $1.04*10^{3}$ & 600 & $36.90\pm8,07$ & $0.37\pm0.098$  \\ \cline{2-6}

\multirow{7}{*}{} & \multirow{1}{*}{GPGN} & $2.20*10^{3}$ & 600 & $43.60\pm5.47$ & $0.44\pm0.072$  \\  \cline{2-6}

\multirow{7}{*}{} & BPGD &  $891.92$ 
& 600 & $74.35\pm17.78$ &  $0.75\pm0.17$ \\ \cline{2-6} 

\multirow{7}{*}{} & IBPG  & $1.20*10^{3}$ & 580.45 & $62.55\pm17.59$   & $0.64\pm0.16$  \\ \cline{2-6}

\multirow{7}{*}{} & BPL  & $1.69*10^{3}$ & 592.92 & $58.55\pm15.04$ & $0.60\pm0.13$   \\ \cline{2-6}

\multirow{7}{*}{} & \multirow{1}{*}{APALM}  & $325.71$ & 472.46 & $93.8\pm4.70$ & $0.94\pm0.044$  \\ \cline{2-6}

\multirow{7}{*}{} & \multirow{1}{*}{APALM$^+$}  & $\bold{325.17}$ & $\bold{353.01}$ & $\bold{94.3\pm4.40}$ & $\bold{0.95\pm0.042}$ \\ \hline

\multirow{7}{*}{Br35H ::} & \multirow{1}{*}{GIST}  & $7.055*10^{3}$ & 2000 & $35.20\pm4.36$ & $0.34\pm0.034$  \\ \cline{2-6}

\multirow{7}{*}{Brain Tumo-} & \multirow{1}{*}{GPGN} & $9.36*10^{4}$ & 2000 & $36.40\pm3.73$ & $0.35\pm0.026$  \\  \cline{2-6}

\multirow{7}{*}{Detection 2020} & BPGD &  $7.19*10^{3}$ 
& 2000 & $53.65\pm5.26$ &  $0.53\pm0.048$ \\ \cline{2-6} 

\multirow{7}{*}{} & IBPG  & $8.12*10^{3}$ & 2000 & $51.45\pm4.36$   & $0.51\pm0.042$  \\ \cline{2-6}

\multirow{7}{*}{} & BPL  & $8.91*10^{3}$ & 2000 & $50.70\pm4.64$ & $0.50\pm0.043$   \\ \cline{2-6}

\multirow{7}{*}{} & \multirow{1}{*}{APALM}  & $351.85$ & 1881.64 & $62.90\pm6.06$ & $0.61\pm0.067$  \\ \cline{2-6}

\multirow{7}{*}{} & \multirow{1}{*}{APALM$^+$}  & $\bold{322.34}$ & $\bold{1687.75}$ & $\bold{63.15\pm6.72}$ & $\bold{0.62\pm0.078}$ \\ \hline 

\end{tabular*}
\label{real}
\end{table*}

Utilizing Table \ref{real} and Figure \ref{figreal} as references, we observed
(i) APALM$^+$ demonstrates a notable ability to uncover the underlying data structure within the matrix and tensor compared to other algorithms. In contrast, traditional sparse logistic regression methods (GPGN and GIST) exhibit limited effectiveness in this aspect; 
(ii) In contrast to other algorithms, APALM$^+$ consistently surpasses them across all evaluation. 
(iii) The impact of APALM$^+$ (adaptive momentum) is better than APALM (non-adaptive);  
(iv) As the sample dimension of the real dataset increases, the iteration speed of GPGN and GIST notably slows down, but the algorithms employing tensor input, such as APALM$^+$, remain remarkably unaffected. This illustrates the effectiveness of tensor-based logistic regression in mitigating the curse of dimensionality.

\section{Conclusion}
In this paper, we present a Multilinear Sparse Logistic Regression ($\ell_0$-MLSR) model with $\ell_0$-constraints, it enables direct predictions using tensor-based input data and ensure solution sparsity through $\ell_0$-constraints. However, the challenge lies in solving the model due to its association with a nonconvex and nonsmooth optimization problem. 
To address the $\ell_0$-MLSR model, we develop an Accelerated Proximal Alternating Linearized Minimization with Adaptive Momentum (APALM$^+$).  
By utilizing APALM$^+$ to solve $\ell_0$-MLSR, we establish the objective function's convergence. By utilizing the Kurdyka-ojasiewicz property, we also establish the global convergence and convergence rate of APALM$^+$. 
The results from numerical experiments provide strong evidence for the effectiveness of our proposed algorithm. 
These findings underscore our algorithm's potential to drive substantial progress in the $\ell_0$-MLSR model, thus showcasing its promising aptitude for advancing multilinear machine learning models.

\begin{appendices}
\section{Proof for Convergence Analysis}
To simplify the following proof process, we first show an equivalent expression of Eq.(4). Mathematically, Eq. (\ref{e23}) is a member of the following family of functions:
\begin{align}
J(x_{(p+1)})=H(x_{(p+1)})+\sum_{i=1}^{p+1} F_{i}(x_{i}).
\label{e62}
\end{align}
where Eq. (\ref{e62}) satisfies Assumption \ref{assump1}. When $x_{(p)}=\mathcal{W},~ b=x_{p+1}$, Eq. (\ref{e62}) degenerates into Eq. (\ref{e23}). 

Let $S$ be the domain of Eq. (\ref{e62}), the proximal operator is defined as: 
\begin{equation}
\begin{aligned}
&&x^{k+1}_{j} &\in \operatorname*{\arg\min} \limits_{x \in S^{d_{j}} \subseteq S} (F_{j}(x_{j})+\frac{1}{2\sigma_{j}^{k}}||x-y^{k}_{j}||^{2} \\
&& \ &+\langle \nabla_{x_{j}} H(\left\{{g_{i}} \right\}^{j-1}_{i=1}, y^{k}_{j}, \left\{{g_{i}} \right\}^{n}_{i=j+1}), x-y^{k}_{j} \rangle).  
\end{aligned}
\label{e63}
\end{equation}
where $\sigma_{j}^{k}=\frac{1}{\tau^{k}_{j}}$ (compare to Eq. (\ref{prox})). 

Since Eq. (\ref{e23}) is just a special case of Eq. (\ref{e62}),  thus the Assumption \ref{assump1}, Eq. (\ref{e62}) and Eq. (\ref{e63}) will serve as the foundation for all subsequent proofs in this appendix.

\subsection{Proof of  Theorem \ref{t1}} \label{TP1}
\begin{proof}
From Proposition \ref{p3},when $n=1$, we know that: 
\begin{align}
&&H(x^{k+1}_{(1)},\left\{y_{i}^{k}\right\}_{i=1}^{P})&\leq \langle \nabla_{x_{1}} H(y^{k}_{(P)}),x_{1}^{k+1}-y^{k}_{1}\rangle \notag\\ 
&&\ &+H(y^{k}_{(P)})+\frac{L_{\nabla_{x_{1}} H}}{2}||x^{k+1}_{1}-y^{k}_{1}|| 
\label{e311}
\end{align}
since Eq. (\ref{e63}), we obtain:
\begin{align}
 && F_{1}(y^{k}_{1}) &\geq F_{1}(x^{k+1}_{1})+\frac{\sigma^{k}_{1}}{2}||x^{k+1}_{1}-y^{k}_{1}||^{2} \notag \\
 &&\  &+\langle \nabla_{x_{1}} H(x^{k+1}_{1},\left\{y_{i}^{k}\right\}_{i=1}^{P}),x_{1}^{k+1}-y^{k}_{1} \rangle.\label{e312}
\end{align}
then sum of the Eq. (\ref{e311}) and Eq. (\ref{e312}), we have:
\begin{align}
 && H(x^{k+1}_{1},\left\{y_{i}^{k}\right\}_{i=1}^{P}) &\leq H(y^{k}_{(P)})+F_{1}(y_{1}^{k+1}) \notag \\
 &&\  &-\rho||x^{k+1}_{1}-y^{k}_{1}||^{2}. \notag
\end{align}
where $\rho=\frac{1}{2\sigma^{k}_{1}}-L_{\nabla_{x_{1}} H}$. Therefore, when $n=1$,  Theorem \ref{t1} is obviously true. Assuming the  Theorem \ref{t1} (\romannumeral1) holds when $n=N$, i.e.
\begin{align}
&& H(x^{k+1}_{(N)},\left\{y_{i}^{k}\right\}_{i=N+1}^{P}))&\leq H(y^{k+1}_{(P)})+\sum_{i=1}^{N}F_{i}(y_{i}^{k+1}) \notag \\
&& &-\rho||x^{k+1}_{(N)}-y^{k}_{(N)}||-\sum_{i=1}^{N}F_{i}(x_{i}^{k+1}). \label{e313}
\end{align}
where $\rho=\min(\left\{\frac{1}{2\sigma^{k}_{i}}-L_{\nabla_{x_{i}} H}
\right\}_{i=1}^{N})$. We hope that is still established when $n=N+1$. Since Eq. (\ref{e63}) and Proposition \ref{p3}, there are some inequalities that hold:
\begin{align}
&& &F_{N+1}(y^{k}_{N+1}) \geq F_{N+1}(x^{k+1}_{N+1}) +\frac{\sigma^{k}_{N+1}}{2}||x^{k+1}_{N+1}-y^{k}_{N+1}||^{2} \notag \\ 
&&   &+\langle \nabla_{x_{N+1}} H(x^{k+1}_{(N+1)},\left\{y_{i}^{k}\right\}_{i=N+2}^{P}),x_{N+1}^{k+1}-y^{k}_{N+1} \rangle. 
\label{e314}
\end{align}    
From Proposition \ref{p3}, we infer: 
\begin{align}
&& &H(x^{k+1}_{(N+1)},\left\{y_{i}^{k}\right\}_{i=N+2}^{P}) \leq H(x^{k+1}_{(N)},\left\{y_{i}^{k}\right\}_{i=N+1}^{P})  \notag \\
&& &+\frac{L_{\nabla_{x_{N+1}} H}}{2}||x^{k+1}_{N+1}-y^{k}_{N+1}|| \notag \\
&& &+ \langle \nabla_{x_{N+1}} H(x^{k+1}_{(N+1)},\left\{y_{i}^{k}\right\}_{i=N+2}^{P}),x^{k+1}_{N+1}-y^{k}_{N+1}\rangle. \label{e315}
\end{align}
then sum of the Eq. (\ref{e313}), Eq. (\ref{e314}) and Eq. (\ref{e315}), when $n=P=p+1$, the  Theorem \ref{t1} (\romannumeral1) is true.

\myspace{0} because J is lower bounded, thus: 
\begin{center}
$\sum_{k=1}^{\infty} ||x^{k+1}_{(P)}-y^{k}_{(P)}||^{2}=J(x_{(P)}^{0})-inf J<\infty.$
\end{center}
which imply $\lim_{k \to \infty} ||x^{k+1}_{(P)}-y^{k}_{(P)}||=0$.
\end{proof}

\subsection{Proof of  Theorem \ref{t2}} \label{TP2}
\begin{proof}
(\romannumeral1) Let: 
\begin{center}
$h(y_{j}^{k})=H(\left\{{x_{i}}^{k+1} \right\}^{j-1}_{i=1}),y_{j}^{k},\left\{{y_{i}}^{k} \right\}^{P}_{i=j+1}))$
\end{center}
by \ref{p1} and Proposition \ref{p2}, it follows that:
\begin{align}
&& 0 &\in  \nabla_{x_{j}}h(y^{k}_{j})-\frac{1}{\sigma^{k}_{j}} (y^{k}_{j}-x^{k+1}_{j})\notag \\
&& &+\partial_{x_{j}} (\sum_{i=1}^{j} F_{i}(x^{k+1}_{i})+\sum_{i=j+1}^{P} F_{i}(y^{k}_{i})) \notag
\end{align}
which imply:\\
\begin{align}
&& &\nabla_{x_{j}} h(x^{k+1}_{j}) -\nabla_{x_{j}}h (y^{k}_{j})+ \frac{1}{\sigma^{k}_{j}}  ||y^{k}_{j}-x^{k+1}_{j} || \notag \\ 
&&\ &\in \nabla_{x_{j}}h (x^{k+1}_{j})+\partial_{x_{j}} (\sum_{i=1}^{j} F_{i}(x^{k+1}_{i})+\sum_{i=j+1}^{P} F_{i}(y^{k}_{i})) \notag \\
&&\ &= \partial_{x_{j}}  J(\left\{{x^{k+1}_{i}}\right\}^{j-1}_{i=1},x^{k+1}_{j},\left\{{y^{k}_{i}} \right\}^{P}_{i=j+1})\notag
\end{align}

(\romannumeral2) Define: 
\begin{align}
g_{x_{j}^{k+1}}=\nabla_{x_{j}}h (x^{k+1}_{j})-\nabla_{y_{j}}h (y^{k}_{j})+\frac{1}{\sigma^{k}_{j}
}||y^{k}_{j}-x^{k+1}_{j}||. \notag
\end{align}
Since Definition \ref{dlipchitz}, ${\forall} j \in \mathbb{N}$, we have: 

\begin{align}
&& ||g_{x_{j}^{k+1}}-g_{y_{j}^{k}}|| &\leq \frac{L_{\nabla_{x_{j}} H}}{2}||x^{k+1}_{j}-y^{k}_{j}||+\frac{1}{2\sigma^{k}_{j}} ||x^{k+1}_{j}-y^{k}_{j}|| \notag \\
&& &=(\frac{L_{\nabla_{x_{1}} H}}{2}+\frac{1}{2\sigma^{k}_{j}}) ||x^{k+1}_{j}-y^{k}_{j}||.  \notag
\end{align}
thus we infer:
\begin{align}
||\left\{g_{x^{k+1}_{i}} \right\}^{p}_{i=1}||=\sum_{i=1}^{P} ||g_{x_{j}^{k+1}}-g_{y_{j}^{k}}||\leq  \rho_{b}||(\left\{{x^{k+1}_{i}}-{y^{k}_{i}}\right\}^{p}_{i=1}||.
\end{align}
where $\rho_{b}=max(\left\{\frac{1}{2\sigma^{k}_{i}}+L_{\nabla_{x_{i}} H}\right\}_{i=1}^{n+1})$. when $P=p+1$, Theorem \ref{t2} is true.
\end{proof}

The definition of $z^{k}$ and $c^{k}$ is the same as the main body text, namely:
\begin{center}
$z^{k}=\left\{x^{k}_{i} \right\}^{P}_{i=1},~c^{k}=\left\{y^{k}_{i} \right\}^{P}_{i=1}$.
\end{center}

\subsection{Proof of   Theorem \ref{c5}} \label{TP3}
\begin{proof}
According to Assumption \ref{assump1}, ${\forall}\overline{z} \in z'$, there exists a subsequence $z^{k_{j}}$  such that: $\lim_{j \to \infty} z^{k_{j}}= \overline{z}$, since Lemma \ref{tf4}, $\overline{z}$ of arbitrariness and Theorem \ref{t1}, we infer: 
  \begin{center}
   $\lim_{j \to \infty} J(z^{k_{j}})=J( \overline{z})=J^{*}$.
    \end{center}
which means J is a constant on z'. Since Theorem \ref{t1} and Theorem \ref{t2}, we have: 
  \begin{center}
   $\lim_{k \to \infty} ||\left\{g_{x^{k+1}_{i}} \right\}^{P}_{i=1}|| \leq \lim_{k \to \infty} \rho_{b}||(\left\{{x^{k+1}_{i}}-{y^{k}_{i}}\right\}^{P}_{i=1}||, $\\
   $\lim_{k \to \infty}||\left\{g_{x^{k+1}_{i}} \right\}^{P}_{i=1}|| = 0$.
    \end{center}
when $P=p+1$, $z' \subseteq crit J$. 
\end{proof}

\subsection{Proof of Theorem \ref{glo}} \label{TP4}
A straightforward application of the methodology cited in \cite{bolte2014proximal} to our algorithms is not feasible because our algorithms (APALM and APALM$^+$) have the presence of extrapolated sequences in a sufficiently descending property. In the following proof, we adapt the proof strategy to make it compatible with the algorithm we provide. 
\begin{proof}
According to Assumption \ref{assump1}, $\left\{z^{k}\right\}$ is bounded and complete. From Definition \ref{d8}, ${\forall} \eta >0$, there exists a positive integer $k_{0}$ such that: $J( \overline{z})<J(z^{k_{0}})<J( \overline{z})+\eta$. Since Definition \ref{d8}, there exists a concave function $\phi$ so that: $ \phi^{'}(J(z^{k})-J( \overline{z})) dist(0, \partial J(z^{k})) \geq 1$, thus we infer: 
\begin{align}
&& dist(0, \partial J(z^{k}) ) &\leq ||\left\{p_{x^{k+1}_{i}} \right\}^{P}_{i=1}||=\rho_{b}||z^{k}-c^{k-1}||. \notag
\end{align}
according to the $ \phi^{'}(J(z^{k})-J( \overline{z})) dist(0, \partial J(z^{k})) \geq 1$, which imply: 
  \begin{center}
  $ \phi^{'}(J(z^{k})-J( \overline{z})) \geq \frac{1}{dist(0, \partial J(z^{k}))}$
  $\geq \frac{1}{\rho_{b}||z^{k}-c^{k-1}||}$.
    \end{center}
let $G(k)=J(z^{k})-J( \overline{z})$, from definition of concave function and Theorem \ref{t1}, we have: 
  \begin{align}
&& \phi(G(k))-\phi(G(k+1)) &\geq \phi^{'}(G(k))(G(k)-G(k+1))  \notag \\
&& &\geq \frac{\rho||z^{k+1}-c^{k}||^{2}}{\rho_{b}||z^{k}-c^{k-1}||}. \notag 
    \end{align}
define C=$\frac{\rho}{\rho_{b}}$, C is a constant, we infer: 
  \begin{center}
    $||z^{k+1}-c^{k}||^{2} \leq C (\phi(G(k))-\phi(G(k+1))) ||z^{k}-c^{k-1}||$.
    \end{center}
Using the fact that $2ab\leq a^{2}+b^{2}$: 
  \begin{center}
    $2 ||z^{k+1}-c^{k}|| \leq C (\phi(G(k))-\phi(G(k+1))) +||z^{k}-c^{k-1}||$.
    \end{center}
sum both sides: 
\begin{align}
&&  2 \sum_{k=l+1}^{K} ||z^{k+1}-c^{k}|| &\leq \sum_{k=l+1}^{K} ||z^{k}-c^{k-1}|| \notag \\
&& &+C(\phi(G(l+1))-\phi(G(K+1))). \notag 
\end{align}
thus from Lemma \ref{tf4} and Assumption \ref{assump1}, we get that: 
\begin{align}
&& \lim_{K \to \infty} \sum_{k=l+1}^{K} ||z^{k+1}-c^{k}|| &\leq ||z^{l+1}-c^{l}|+ C\phi(G(l+1)).  \notag \\
&& &-\lim_{K \to \infty} C \phi(G(K+1))) \notag \\
 && & < \infty .
 \label{c7}
\end{align}
from Assumption \ref{assump1}, no matter $c^{k} =z^{k}+\beta_{k} (z^{k}-z^{k-1})$ or $c^{k} =z^{k}$, we always have: 
  \begin{center}
    $ \ ||z^{k+1}-z^{k}|| -\beta_{max} ||z^{k}-z^{k-1}||  \
  \leq||z^{k+1}-c^{k}|| $.
    \end{center}
let $s_{k}=||z^{k+1}-z^{k}||$, from Eq. (\ref{c7}), we know that $\sum_{k=K}^{\infty}( s_{k+1} -\beta_{max} s_{k})<\infty$, thus we infer: 
 \begin{align}
&&\sum_{k=K+1}^{\infty} (1-\beta_{max})s_{k}-\beta_{max} s_{K} &=\sum_{k=K}^{\infty}( s_{k+1} -\beta_{max} s_{k}). \notag 
\end{align}
from  $(1-\beta_{max})s_{k}>0$ and $ (1-\beta_{max})$ is a constant, thus: 
    \begin{align}
    &&\lim_{K\rightarrow \infty} ||z^{K+s}-z^{K}|| \leq \lim_{K\rightarrow \infty}\sum_{k=K+1}^{\infty}s_{k}=0.
    \end{align}
which means the Theorem \ref{glo} is true. 
\end{proof}

\subsection{Proof of Theorem \ref{rate}}\label{TP5}
\begin{proof}
Checking the assumptions of Theorem 2 in reference \cite{li2017convergence}, we observe that all assumptions required in our algorithm are clearly satisfied, so the theorem holds. 
\end{proof}
\end{appendices}

\section*{Acknowledgment}
The work of W. Y and W. M. was supported in part by the National Natural Science Foundation of China (62262069), in part by the Program of Yunnan Key Laboratory of Intelligent Systems and Computing (202205AG070003), in part by the Yunnan Fundamental Research Projects under Grant (202201AT070469, 202301BF070001-019). 

\balance
\small{
\bibliographystyle{IEEEtran}
\bibliography{refer}
}
\end{document}